\documentclass{colt2015} 

\title[First-order regret bounds for combinatorial semi-bandits]{First-order regret bounds for combinatorial
semi-bandits}
\usepackage{times}

 \usepackage{amsmath}
\usepackage{amssymb}
\usepackage{xspace}
\usepackage{bm}
\usepackage{algorithm}

\newlength{\minipagewidth}
\newcommand{\loss}{\ell}
\newcommand{\hloss}{\wh{\ell}}
\newcommand{\tloss}{\wt{\ell}}

\newcommand{\real}{\mathbb{R}}

\newcommand{\Sw}{\mathcal{S}}

\newcommand{\OO}{\mathcal{O}}
\newcommand{\tOO}{\wt{\OO}}

\newcommand{\PP}[1]{\mathbb{P}\left[#1\right]}
\newcommand{\EE}[1]{\mathbb{E}\left[#1\right]}

\newcommand{\PPcc}[2]{\mathbb{P}\left[\left.#1\right|#2\right]}

\newcommand{\EEcc}[2]{\mathbb{E}\left[\left.#1\right|#2\right]}

\newcommand{\ev}[1]{\left\{#1\right\}}
\newcommand{\pa}[1]{\left(#1\right)}
\newcommand{\bpa}[1]{\bigl(#1\bigr)}
\newcommand{\F}{\mathcal{F}}

\newcommand{\tp}{\wt{p}}
\newcommand{\tq}{\wt{q}}

\renewcommand{\th}{\ensuremath{^{\mathrm{th}}}}
\def\argmin{\mathop{\rm arg\, min}}
\newcommand{\hL}{\wh{L}}

\newcommand{\bV}{\boldsymbol{V}}

\newcommand{\bw}{\boldsymbol{w}}

\newcommand{\bu}{\boldsymbol{u}}
\newcommand{\bv}{\boldsymbol{v}}

\newcommand{\bz}{\boldsymbol{z}}
\newcommand{\bloss}{\bm{\ell}}

\newcommand{\bL}{\boldsymbol{L}}

\newcommand{\tbZ}{\widetilde{\bZ}}

\newcommand{\tbV}{\widetilde{\bV}}

\newcommand{\tV}{\widetilde{V}}
\newcommand{\hbl}{\wh{\bloss}}

\newcommand{\hbL}{\wh{\bL}}

\newcommand{\bZ}{\boldsymbol{Z}}

\newcommand{\wh}{\widehat}
\newcommand{\wt}{\widetilde}

\newcommand{\fpl}{{FPL}\xspace}
\newcommand{\fpltrix}{\textsc{FPL-TrIX}\xspace}

\newcommand{\exph}{\textsc{Exp3}\xspace}

\newcommand{\hedge}{\textsc{Hedge}\xspace}
\newcommand{\green}{\textsc{Green}\xspace}
\newcommand{\osmd}{\textsc{OSMD}\xspace}
\newcommand{\scrible}{\text{SCRiBLe}\xspace}
\newcommand{\ra}{\rightarrow}
\newcommand{\transpose}{^\mathsf{\scriptscriptstyle T}}

\newcommand{\norm}[1]{\left\|#1\right\|}
\newcommand{\onenorm}[1]{\norm{#1}_1}

\usepackage{todonotes}
\definecolor{PalePurp}{rgb}{0.66,0.57,0.66}

 \coltauthor{\Name{Gergely Neu} \Email{gergely.neu@gmail.com}\\
 \addr INRIA Lille -- Nord Europe
 \\ 40 avenue Halley, Villeneuve d'Ascq, 59650, France
 }

\begin{document}

\maketitle

\begin{abstract} 
We consider the problem of online combinatorial optimization under semi-bandit feedback, where a learner has to 
repeatedly pick actions from a combinatorial decision set in order to minimize the total losses associated with its 
decisions. After making each decision, the learner observes the losses associated with its action, but not other 
losses. For this problem, there are several learning algorithms that guarantee that the learner's expected regret grows 
as $\tOO(\sqrt{T})$ with the number of rounds $T$. In this paper, we propose an algorithm that improves this scaling to 
$\tOO(\sqrt{\smash[b]{L_T^*}})$, where $L_T^*$ is the total loss of the best action. 
Our algorithm is among the first to achieve such guarantees in a partial-feedback scheme, and the first one to do so in 
a combinatorial setting. 
\end{abstract}

\begin{keywords}
online learning, online combinatorial optimization, semi-bandit feedback, follow the perturbed leader, improvements for 
small losses, first-order bounds
\end{keywords}

\section{Introduction}
Consider the problem of sequential multi-user channel allocation in a cognitive radio network (see, e.g.,
\citealp{GKJ12}). In this problem, a network operator sequentially matches a set of $N$ \emph{secondary users} to a set
of $M$ \emph{channels}, with the goal of maximizing the overall quality of service (QoS) provided for the secondary
users, while not interfering with the quality provided to \emph{primary users}. Due to different QoS preferences of
users and geographic dispersion, different users might perceive the quality of the same channel differently.
Furthermore, due to uneven traffic on the channels and other external conditions, the quality of each matching may 
change
over time in a way that is very difficult to model by statistical assumptions. Formally, the \emph{loss} associated 
with user
$i$ being matched to channel $j$ in the $t$\th~decision-making round is $\loss_{t,(ij)}\in[0,1]$, and the 
goal of the network
operator is to sequentially select matchings $\bV_t$ so as to minimize its total loss $\sum_{t=1}^T
\sum_{(ij)\in\bV_t}\loss_{t,(ij)}$ after $T$ rounds. It is realistic to assume that the operator learns about the
instantaneous losses of the allocated user-channel pairs after making each decision, but counterfactual losses are never
revealed. 

Among many other sequential optimization problems of practical interest such as sequential routing or online
advertising, the above problem can be formulated in the general framework of \emph{online combinatorial optimization}
\citep{audibert13regret}. This learning problem can be formalized as a repeated game between a \emph{learner} and an
\emph{environment}. In every round $t=1,2,\dots,T$, the learner picks a decision $\bV_t$ from a
combinatorial decision set $\Sw\subseteq\ev{0,1}^d$. 
Simultaneously, the environment fixes a loss vector $\bloss_t\in[0,1]^d$ and the learner suffers a loss of
$\bV_t\transpose\bloss_t$. We assume that $\onenorm{\bv}\le m$ holds for all $\bv\in\Sw$, entailing
$\bV_t\transpose\bloss_t\le m$. At the end of the round, the learner observes some feedback based on $\bV_t$ and 
$\bloss_t$.
The simplest setting imaginable is called the \emph{full-information} setting where the learner observes the entire 
loss vector
$\bloss_t$. In most practical situations, however, the learner cannot expect such rich feedback. In this paper, we focus
on a more realistic and challenging feedback scheme known as \emph{semi-bandit}: here the learner observes the subset of
components $\loss_{t,i}$ of the loss vector with $V_{t,i} = 1$. Note that this precise feedback scheme arises 
in our cognitive-radio example.
The performance of the learner is measured in terms of the \emph{regret}
\[
 R_T = \max_{\bv\in\Sw} \sum_{t=1}^T \pa{\bV_t - \bv}\transpose\bloss_t,
\]
that is, the gap between the total loss of the learner and that of the best fixed action. The interaction history up to 
time $t$ is captured by $\F_{t-1}=\sigma(\bV_1,\dots,\bV_{t-1})$. In the current paper, we focus 
on \emph{oblivious} environments who are only allowed to pick each loss vector $\bloss_t$ independently of $\F_{t-1}$. 
The learner is allowed to (and, by standard arguments, should) randomize its decision $\bV_t$ based on the observation 
history $\F_{t-1}$. With these remarks in mind, we will focus on the \emph{expected regret} $\EE{R_T}$ from now on, 
where the expectation integrates over the randomness injected by the learner.

Most of the literature is concerned with finding algorithms for the learner that guarantee that the regret grows as
slowly as possible with $T$. Of equal importance is establishing lower bounds on the learner's regret against
specific classes of environments. Both of these questions are by now very well-studied, especially in the simple case
where $\Sw$ is the set of $d$-dimensional unit vectors; this setting is known as \emph{prediction with expert advice}
when considering full feedback (e.g., \citealp{CBLu06:book}) and the \emph{multi-armed bandit} problem when considering
semi-bandit feedback (e.g., \citealp{auer2002bandit}). 
In these settings, the minimax regret is known to be of
$\Theta(\sqrt{T\log d})$ and $\Theta(\sqrt{dT})$, respectively. Several learning algorithms are known to achieve these
regret bounds, at least up to logarithmic factors in the bandit case, with the notable exception of the \textsc{PolyINF}
algorithm proposed by \citet{AB09}.
The minimax regret for the general combinatorial setting was studied by \citet{audibert13regret}, who show that no
algorithm can achieve better regret than $\Omega(m\sqrt{T\log (d/m)})$ in the full-information setting, or
$\Omega(\sqrt{mdT})$ in the semi-bandit setting. \citeauthor{audibert13regret}~also propose algorithms that achieve
these guarantees under both of the above feedback schemes. Furthermore, they show that a natural (although 
not always efficient) extension of the \exph strategy of \citet{auer2002bandit} guarantees a regret bound of 
$\OO\bpa{m\sqrt{dT\log (d/m)}}$ in the semi-bandit setting (see also \citealp{gyorgy07sp}). A computationally 
efficient strategy for the same setting was proposed by \citet{NeuBartok13}, who
show that an augmented version of the FPL algorithm of \citet{KV05} achieves a regret of $\OO(m\sqrt{dT\log d})$,
essentially matching the bound of \exph.

Even though the above guarantees cannot be substantially improved under the worst possible realization of the loss
sequence, certain improvements are possible for specific types of loss sequences. Arguably, one of the most fundamental
of
these improvements are bounds that replace the number of rounds $T$ with the loss of the best action $L_T^* = 
\min_{\bv\in\Sw}\bv\transpose\bL_T$, thus guaranteeing a regret of  $\tOO(\sqrt{\smash[b]{L_T^*}})$. Such improved 
bounds, often
called \emph{first-order} regret bounds, are abundant in the online learning literature \emph{when assuming full
feedback}: the key for obtaining such results is usually a clever tuning rule for otherwise standard learning
algorithms such as \hedge \citep{CBMS:COLT2005,CBLu06:book} or FPL \citep{HuPo04,KV05,EWK14}. The intuitive advantage 
of such first-order bounds that they can effectively take advantage of ``easy'' learning problems where there exists an 
action with superior performance. In our cognitive-radio example, this corresponds to the existence of a user-channel 
matching that tends to provide high quality of service.

One obvious question is whether such improvements are possible under partial-information constraints. We can
answer this question in the positive, although such bounds are far less common than in the full information case. In
fact, we are only aware of three algorithms that achieve such bounds: \textsc{Exp3Light} described in Section~4.4 
of \citet{stoltz05thesis}, \green by \citet{allenberg06hannan} and \scrible by \citet{AHR12}, as shown by 
\citet{RS13}\footnote{The obscure nature of such first-order bounds is reflected by the 
fact that \citeauthor{RS13} prove their corresponding result simply because they were not aware of the two previous 
results.}.
These algorithms guarantee regret bounds of $\OO(d\sqrt{\smash[b]{L_T^*\log d}})$, $\OO(\sqrt{d\smash[b]{L_T^*\log 
d}})$ and $\OO(d^{3/2}\sqrt{\smash[b]{L_T^*\log (dT)}})$ in the 
multi-armed bandit problem, respectively. These results, however, either do not generalize to the combinatorial 
setting (see Section~\ref{sec:close} for a discussion on \green) or already scale poorly with the problem size in 
the simplest partial-information setting. Furthermore, implementing these algorithms is also not straightforward for 
combinatorial  decision sets.

In this paper, we propose a computationally efficient algorithm that guarantees similar improvements for 
combinatorial semi-bandits. Our approach is based on the Follow-the-Perturbed-Leader (FPL) algorithm of \citet{Han57}, 
as popularized by \citet{KV05}. We show
that an appropriately tuned variant of our algorithm guarantees a regret bound of
$\OO\bigl(m\sqrt{\smash[b]{dL_T^*\log(d/m)}}\bigr)$, largely improving on the minimax-optimal bounds whenever $L_T^*= 
o(T)$. In 
the case of multi-armed bandits where $m=1$, the bound becomes $\OO(\sqrt{\smash[b]{dL_T^*\log d}})$. Notice
however that when $m>1$, $L_T^*$ can be as large as $\Omega(mT)$ in the worst case, making our bounds inferior to the 
best known bounds concerning \fpl and \exph. To circumvent this problem, as well as the need to know a bound on 
$L_T^*$ to tune our parameters, we also propose an adaptive variant of our algorithm that guarantees a regret of 
$\OO\bigl(m\sqrt{\smash[b]{\min\{dL_T^*,dT\}\log (d/m)}}\bigr)$. 
Thus, our performance guarantees are in some sense the strongest among known results for non-stochastic combinatorial 
semi-bandits. 

Besides first-order bounds, there are several other known ways of improving worst-case performance guarantees of 
$\tOO(\sqrt{T})$ for non-stochastic multi-armed bandits. A common improvement  is  replacing $T$ by the \emph{gain} of 
the best action, $T - L_T^*$ (see, e.g., \citealp{auer2002bandit,AB09}). Such bounds,
while helpful in some cases where \emph{all} actions tend to suffer large losses (e.g., in online advertising where even
the best ads have low clickthrough rates), are not as satisfactory as our bounds: these bounds get
worse and worse as one keeps increasing the gain of the best action, even if all other losses are kept constant, 
despite the intuition that this operation actually makes the learning problem much easier. That
is, bounds of the above type fail to reflect the ``hardness'' of the learning problem at hand. The work of 
\citet{HK11} considers a much more valuable type of improvement: they provide regret bounds of $\tOO(d^2
\sqrt{\smash[b]{Q_T}})$, where $Q_T = \min_{\bm{\mu}\in\real^d} \sum_{t=1}^T\norm{\bloss_t - \bm{\mu}}_2^2$ is the 
\emph{quadratic
variation} of the losses. Such bounds are very strong in situations where the sequence of loss vectors ``stays close''
to its mean in all rounds.  Notice however that, unlike our first-order bounds, this improvement requires a condition
to hold for \emph{entire loss vectors} and not just the loss of the best action. This implies that
first-order bounds are more robust to loss variations of obviously suboptimal actions. 
On the other hand, it is also easy to construct an example where $L_T^*$ grows linearly while $Q_T$ is zero. In
summary, we conclude that first-order bounds and bounds depending on the quadratic variation are not comparable
in general, as they capture very different kinds of regularities in the loss sequences. For further discussion of
higher-order and variation-dependent regret bounds, see \citet{CBMS:COLT2005} and \citet{HK10}. We also mention that 
several other types of improvements exist for full-information settings---we refer to recent works of \citet{RS13}, 
\citet{SNL14} and the references therein.

Finally, let us comment on related work on the so-called \emph{stochastic} bandit setting where the loss vectors are
drawn i.i.d.~in every round. In this setting, combinatorial semi-bandits have been studied under the name
``combinatorial bandits'' \citep{GKJ12,CWY13}, giving rise to a bit of confusion\footnote{The term ``combinatorial
bandits'' was first used by \citet{CeBiLu09}, in reference to online combinatorial optimization problems under full 
bandit feedback where the learner only observes $\bV_t\transpose\bloss_t$ after round $t$.}. This line of work focuses 
on 
proving bounds on the \emph{pseudo-regret} defined as $\max_{\bv\in\Sw}
\sum_{t=1}^T \pa{\bV_t - \bv}\transpose \bm{\mu}$, where $\bm{\mu}\in\real^d$ is the mean of the random vector
$\bloss_1$. We highlight the result of \citet{KWASz15}, who have very recently proposed an algorithm that guarantees
bounds on the pseudo-regret of $\OO(md(1/\Delta)\log T)$ for some distribution-dependent constant $\Delta>0$ and a
worst-case bound of $\OO(\sqrt{mdT\log T})$. Note however that comparing these pseudo-regret bounds to bounds on the
expected regret can be rather misleading. In fact, a simple argument along the lines of Section~9 of
\citet{AB10} shows that even algorithms with \emph{zero} pseudo-regret can actually suffer an expected regret of
$\Omega(\sqrt{T})$, when permitting multiple optimal actions. A more refined argument shows that this bound can be
tightened to $\Omega(\sqrt{\smash[b]{L_T^*}})$ when assuming non-negative losses, suggesting that first-order bounds on 
the
expected regret are in some sense unbeatable even in a distribution-dependent setting\footnote{\citet{HK10} use a
similar argument to show that variation-dependent bounds are unbeatable for signed losses in a similar sense.}. 

\section{From zero-order to first-order bounds: Keeping the loss estimates close together}\label{sec:close}
We now explain the key idea underlying our analysis. Our approach is based on the observation that regret bounds
for many known bandit algorithms (such as \exph by \citealt{auer2002bandit}, \osmd with relative-entropy 
regularization by \citealt{audibert13regret}, and 
the bandit FPL analysis of \citealt{NeuBartok13}) take the form
\begin{equation}\label{eq:regbound}
 \eta \sum_{t=1}^T \sum_{i=1}^d \loss_{t,i}\cdot \hloss_{t,i} + \frac{D}{\eta} \le 
 \eta \sum_{i=1}^d \hL_{T,i} + \frac{D}{\eta},
\end{equation}
where $\hloss_{t,i}$ is an estimate of the loss $\loss_{t,i}$, $\hL_{T,i} = \sum_{t=1}^T \hloss_{t,i}$, $\eta>0$ is a
tuning parameter, and $D>0$ is a constant that depends on the particular algorithm and
the decision set. The standard approach is 
then to design the loss estimates to be \emph{unbiased} so that the above bound becomes $\eta \sum_{i=1}^d L_{T,i} +
\frac{D}{\eta}$ after taking expectations. Unfortunately, this form does not permit proving first-order bounds as 
$L_{T,i}$ may
very well be $\Omega(T)$ for either $i$ even in very easy problem instances---that is, even an optimized setting of 
$\eta$ gives a regret bound of $\OO(\sqrt{dDT})$ at best. Applying a similar line of reasoning, one can replace $T$ in 
the above bound by $(T-L_T^*)$, the largest total \emph{gain} associated with any component, but, as already discussed 
in the introduction, this improvement is not useful for our purposes.

In this paper, we take a different approach to optimize bounds of the form~\eqref{eq:regbound}. The idea is 
to construct a loss-estimation scheme that keeps every $\hL_{T,i}$ ``close'' to $\hL_T^* = \min_{\bv\in\Sw} \bv^\top 
\hbL_T$, the estimate of the optimal action
in the sense that
\begin{equation}\label{eq:lossbound}
 \hL_{T,i} \le \hL_T^* + \tOO\pa{\frac{1}{\eta}}.
\end{equation}
Observe that this property allows rewriting the bound~\eqref{eq:regbound} as $\eta d \hL_{T}^* + \frac{D}{\eta} +
\tOO(1)$.
Of course, a loss-estimation scheme guaranteeing the above property has to come at the price of a certain bias. 
Guaranteeing that the bias satisfies certain properties and is \emph{optimistic} in the sense that $\mathbb{E}\hL_T^* 
\le L_T^*$, we can arrive at a first-order bound by choosing $\eta = \wt\Theta(\sqrt{1/L_T^*})$. The remaining 
challenge is to come up with an adaptive learning-rate schedule that achieves such a bound without prior knowledge of 
$L_T^*$.

Our approach is not without a precedent: \citet{allenberg06hannan} derive a first-order bound for multi-armed bandits 
based on very similar principles. Their algorithm, called \green, relies on a clever trick that prevents picking arms 
that seem suboptimal. Specifically, \green maintains a set of weights $w_{t,i}$ over the arms and computes an 
auxiliary probability distribution $\tp_{t,i}\propto w_{t,i}$. The true sampling distribution over the arms is 
computed by setting $p_{t,i} = 0$ for all arms such that $\tp_{t,i}$ is below a certain threshold $\gamma$, and 
then redistributing the removed weight among the remaining arms proportionally to $w_{t,i}$. The intuitive effect of 
this thresholding operation is that poorly performing arms are eliminated, which harnesses the further growth 
of their respective estimated losses. Specifically, \citeauthor{allenberg06hannan}~show that  
property~\eqref{eq:lossbound} and $\mathbb{E}\hL_T^* \le L_T^*$ simultaneously hold for their algorithm, paving the way 
for their first-order bound.

While providing strong technical results, \citet{allenberg06hannan} give little intuition as to why this 
approach is key to obtaining first-order bounds and how to generalize their algorithm to more complicated problem
settings such as ours. 
Even if one is able to come up with a generalization on a conceptual level, efficient implementation of such a variant
would only be possible on a handful of decision sets where \exph can be implemented in the first place (see, e.g., 
\citealp{koolen10comphedge,CL12}). The probabilistic nature of the approach of 
\citeauthor{allenberg06hannan} does not seem to mix well with the mirror-descent type algorithms of 
\citet{audibert13regret} either, whose proofs rely on tools from convex analysis. 
In the current paper, we propose an alternative way to restrict sampling of suboptimal actions that leads to
property~\eqref{eq:lossbound} in a much more transparent and intuitive way.

\section{The algorithm: FPL with truncated perturbations and implicit exploration}\label{sec:alg} 
Our algorithm is a variant of the well-known Follow-the-Perturbed-Leader (\fpl) learning algorithm
\citep{Han57,KV05,HuPo04,NeuBartok13}, equipped with a perturbation scheme that will enable us to prove first-order
bounds through guaranteeing property~\eqref{eq:lossbound}. In every round $t$, \fpl chooses its action as
\begin{equation}\label{eq:fpl}
  \bV_t = \argmin_{\bv\in\Sw} \bv\transpose\pa{\eta_t \hbL_{t-1} - \bZ_t},
\end{equation}
where $\eta_t>0$ is a parameter of the algorithm, $\hbL_{t-1}$ is a vector serving as an estimate of the cumulative
loss vector $\bL_{t-1} = \sum_{s=1}^{t-1} \bloss_s$ and $\bZ_t\in\real^d$ is a vector of random perturbations. 
\fpl is very well-studied in the full-information case 
where one can choose $\hbL_{t-1} = \bL_{t-1}$; several perturbation schemes are known to work well in this setting 
\citep{KV05,rakhlin12rr,devroye13rwalk,EWK14,ALTS14}. In what follows, we focus on \emph{exponentially distributed} 
perturbations, which is the only scheme known to achieve near-optimal performance guarantees under bandit feedback 
\citep{Pol05,NeuBartok13}.

In order to guarantee that the condition~\eqref{eq:lossbound} is satisfied, we propose to suppress suboptimal actions 
by using \emph{bounded-support} perturbations. Specifically, we propose to use a \emph{truncated exponential 
distribution} 
with the following density function:
\[
 f_B(z) =
 \begin{cases}
  \frac{e^{-z}}{1 - e^{-B}} &\mbox{, if $z\in[0,B]$}
  \\
  0 &\mbox{ otherwise.}
 \end{cases}
\]
Here, $B>0$ is the bound imposed on the perturbations. In each round $t$, our \fpl variant draws components of the 
perturbation vector $\bZ_t$ independently from an exponential distribution truncated at $B_t>0$, another tuning
parameter of our algorithm.
To define our loss estimates, let us define $q_{t,i} = \EEcc{V_{t,i}}{\F_{t-1}}$ and the vector $\hbl_t$ with components
\begin{equation}\label{eq:lossest}
 \hloss_{t,i} = \frac{\loss_{t,i} V_{t,i}}{q_{t,i} + \gamma_t},
\end{equation}
where $\gamma_t>0$ is the so-called \emph{implicit exploration} (or IX) parameter of the algorithm controlling the bias 
of the loss estimates. Notice that 
$\mathbb{E}\hloss_{t,i} \le \loss_{t,i}$ holds by construction for all $i$. Then, $\hbL_{t}$ is 
simply defined as $\hbL_t = \sum_{s=1}^t \hbl_s$. In what follows, we refer to our algorithm as \fpltrix, standing for
``\fpl with truncated perturbations and implicit exploration''. Pseudocode for \fpltrix is presented as 
Algorithm~\ref{alg}.

\begin{algorithm}
 \textbf{Parameters:} Learning rates $\pa{\eta_t}$, implicit exploration parameters $\pa{\gamma_t}$, truncation 
parameters
$\pa{B_t}$.
\\
 \textbf{Initialization:} $\hbL_0 = 0$.
 \\
 \textbf{For $t=1,2,\dots,T$, repeat}
 \begin{enumerate}
  \item Draw perturbation vector $\bZ_t$ with independent components $Z_{t,i}\sim f_{B_t}$.
  \item Play action
  \[
   \bV_t = \argmin_{\bv\in\Sw} \bv\transpose\pa{\eta_t \hbL_{t-1} - \bZ_t}.
  \]
  \item For all $i$, observe losses $\loss_{t,i}V_{t,i}$ and compute $\hloss_{t,i} =
\frac{\loss_{t,i}V_{t,i}}{q_{t,i} + \gamma_t}$.
  \item Set $\hbL_{t} = \hbL_{t-1} + \hbl_t$.
 \end{enumerate}
 \caption{\fpltrix}\label{alg}
\end{algorithm}

It will also be useful to introduce the notations $D = \log(d/m) +1$ and  $\beta_t = e^{-B_t}$.
For technical reasons, we are going to assume that the sequence of learning rates $\pa{\eta_t}_t$, exploration
parameters $\pa{\gamma_t}_t$ and truncation parameters $\pa{\beta_t}_t$ are all nonincreasing.

Before proceeding, a few comments are in order. First, note that the probabilities $q_{t,i}$ are generally not 
efficiently computable in closed form. This issue can be circumvented by the simple and efficient loss-estimation method 
proposed by 
\citet{NeuBartok13} that produces equivalent estimates on expectation; we resort to the loss 
estimates~\eqref{eq:lossest} to preserve clarity of presentation. Otherwise, similarly to other \fpl-based methods, 
\fpltrix can be efficiently implemented as long as the learner has access to an efficient linear-optimization 
oracle over $\Sw$. Second, we remark that loss estimates of the form~\eqref{eq:lossest} were first proposed by
\citet{KNVM14} as an effective way to trade off the bias and variance of importance-weighted estimates. Finally, one 
may ask if the truncations we introduce are essential for our algorithm to 
work. Answering this question requires a little deeper technical understanding of \fpltrix than the reader might have 
at 
this point, and thus we defer this discussion to Section~\ref{sec:why}. (For the impatient reader, the short answer is 
that one can get away without truncations at the price of an additive $\OO(\log T)$ term in the bounds. Note however 
that the proof of this result still relies on the analysis of \fpltrix that we present in this paper.)

\subsection{Some properties of \fpltrix}\label{sec:props}
In this section, we present some key properties of our algorithm.
We first relate the predictions of \fpltrix to those of an \fpl instance that employs standard (non-truncated)
exponential perturbations. Specifically, we study the relation between the expected performance of \fpltrix that
selects the action sequence $\pa{\bV_t}$ and an auxiliary algorithm that uses a \emph{fixed} exponentially-distributed
perturbation vector $\tbZ$, and plays
\begin{equation}\label{eq:aux}
  \tbV_t = \argmin_{\bv\in\Sw} \bv\transpose\pa{\eta_t \hbL_{t-1} - \tbZ}
\end{equation}
in round $t$. In particular, we are interested in the relation between the quantities
\begin{eqnarray*}
 &p_t(\bv) = \PPcc{\bV_t=\bv}{\F_{t-1}}, &\tp_t(\bv) = \PPcc{\tbV_t=\bv}{\F_{t-1}},
 \\
 &q_{t,i} = \EEcc{V_{t,i}}{\F_{t-1}}, &\tq_{t,i} = \EEcc{\tV_{t,i}}{\F_{t-1}}
\end{eqnarray*}
defined for all $t$, $i$ and $\bv$.
The following lemma establishes a bound on the total variation distance between the distributions induced by $\bZ$ and 
$\tbZ$, and thus relates the above quantities to each other.
\begin{lemma}\label{lem:ptrunc}
Let the components of $\bZ$ and $\tbZ$ be drawn independently from $f_{B_t}$ and $f_\infty$, respectively. Then, for 
any function $G:\real\ra[0,1]$, we have
$\Bigl|\mathbb{E}G(\bZ) - \mathbb{E}G(\tbZ)\Bigr| \le \beta_t d$. In particular, this implies that $\left|p_t(\bv) - 
\tp_t(\bv)\right| \le \beta_t d$ for all $t$ and $\bv$ and $\left|q_{t,i} - \tq_{t,i}\right| \le 
\beta_t d$ for all $t$ and $i$.
\end{lemma}
\begin{proof}\newcommand{\tg}{\wt{g}}
For ease of notation, define $f=f_\infty$, $g = \mathbb{E}G(\bZ)$ and $\tg = \mathbb{E}G(\tbZ)$.
We first prove $g \le \tg + \beta_t d$.
To this end, observe that by the definition of $f_{B_t}$,
\[
\begin{split}
g
&=\int\limits_{\bz\in[0,B_t]^d} G(\bz) f_{B_t}(\bz) \,d\bz
\le \frac{1}{\pa{1-e^{-B_t}}^d} \cdot\int\limits_{\bz\in[0,\infty]^d} G(\bz) f(\bz) \,d\bz
= \frac{\tg}{\pa{1-e^{-B_t}}^d}.
\end{split}
\]
After reordering and using the inequality $\pa{1-x}^d \ge 1-dx$ that holds for all $x \le 1$ and all $d\ge 1$, we obtain
$ g (1-\beta_t d) \le g$.
The upper bound on $g$ follows from reordering again and using $g\le 1$.

To prove the lower bound on $g$, we can use a similar argument as
\[
\begin{split}
g
&=\int\limits_{\bz\in[0,B_t]^d} G(\bz) f_{B_t}(\bz) \,d\bz
= \frac{1}{\pa{1-e^{-B_t}}^d} \cdot\int\limits_{\bz\in[0,B_t]^d} G(\bz) f(\bz) \,d\bz
\\
&\ge \frac{1}{\pa{1-e^{-B_t}}^d} \cdot\biggl(\tg - \int\limits_{\bz\in[B_t,\infty)^d} f(\bz) \,d\bz\biggr)
= \frac{\tg}{\pa{1-e^{-B_t}}^d} - \frac{1 - \pa{1-e^{-B_t}}^d}{\pa{1-e^{-B_t}}^d}.
\end{split}
\]
After reordering and using $\pa{1-x}^d \ge 1-dx$ again, we obtain
\[
\tg
\le g \pa{1-e^{-B_t}}^d + \pa{1 - \pa{1-e^{-B_t}}^d}
\le g + \beta_t d,
\]
concluding the proof.
\end{proof}

The other important property of \fpltrix that we highlight in this section is that the loss estimates generated by the
algorithm indeed satisfy property~\eqref{eq:lossbound}.
\begin{lemma}\label{lem:lossbound}
Assume that the sequences $\pa{\eta_t}$, $\pa{\gamma_t}$ and $\pa{\beta_t}$ are nonincreasing. Then for any $i$ and 
$\bv\in\Sw$, we have
 \[
  \hL_{T,i} \le \bv\transpose \hbL_T + \frac{m\pa{D + B_T}}{\eta_T} + \frac {1}{\gamma_T}. 
 \]
\end{lemma}
\begin{proof}
Fix an arbitrary $i$ and $\bv$ and let $\tau$ denote the last round in which $q_{t,i}>0$. This entails
that $\hL_{T,i} = \hL_{\tau,i}$ holds almost surely, as $V_{t,i} = 0$ for all $t>\tau$. 
 By the construction of the algorithm and the perturbations, $q_{\tau,i}>0$ implies that there exists a $\bw$ with $w_i 
= 1$ and $p_t(\bw)>0$. Thus,
 \[
 \begin{split}
  \bw\transpose\hbL_{\tau-1} \le& 
  \min_{\bu\in\Sw}\bu\transpose\hbL_{\tau-1} + \frac{B_\tau m}{\eta_\tau} \le 
  \tbV_\tau\transpose\hbL_{\tau-1} + \frac{B_Tm}{\eta_T} 
  \\
  =&  \tbV_\tau\transpose\pa{\hbL_{\tau-1} - \frac{1}{\eta_t}\tbZ} + \frac{1}{\eta_t}\tbV_t\transpose\tbZ + 
\frac{B_Tm}{\eta_T} \le 
\bv\transpose\pa{\hbL_{\tau-1} - \frac{1}{\eta_t}\tbZ} + \frac{\tbV_\tau\transpose\tbZ  + B_Tm}{\eta_T},
  \end{split}
 \]
 where the first inequality follows from the fact that $p_t(\bw)>0$, the second one follows from $B_\tau/\eta_\tau \le 
B_T/\eta_T$ and the last one from the definition of $\tbV_\tau$.
 After integrating both sides with respect to the distribution of $\tbZ$ and bounding $\hloss_{\tau,i}\le
1/\gamma_\tau\le 1/\gamma_T$, we obtain the result as
 \[
 \begin{split}
  \hL_{T,i} =& \hL_{\tau-1,i} + \hloss_{\tau,i} \le \bw\transpose\hbL_{\tau-1} + \frac {1}{\gamma_T}
  \le \bv\transpose\hbL_{T} + \frac{m\pa{D + B_T}}{\eta_T} + \frac {1}{\gamma_T},
 \end{split}
 \]
 where we used the fact that $\hL_{t,j}$ is nonnegative for all $j$, $w_i =1$, and $\mathbb{E}\bigl[\tbV_\tau\transpose 
\tbZ\bigr]\le 
m\pa{\log(d/m) + 1} = mD$, which follows from Lemma~\ref{lem:expbound} stated and proved in the Appendix.
\end{proof}

\section{Regret bounds}\label{sec:main}
This section presents our main results concerning the performance of \fpltrix under
various parameter settings.
We begin by stating a key theorem.
\begin{theorem}\label{thm:key}
Assume that the sequences $\pa{\eta_t}$, $\pa{\gamma_t}$ and $\pa{\beta_t}$ are nonincreasing and $\beta_t d\le 
\gamma_t$ holds for all $t$. Then for all $\bv\in\Sw$, the total loss suffered by \fpltrix satisfies
  \[
  \begin{split}
  \sum_{t=1}^T \bV_t\transpose \bloss_t \le& \bv\transpose\hbL_T + 
  \frac{mD}{\eta_T} + \sum_{t=1}^T\pa{\eta_t m + \beta_t d + \gamma_t}\cdot\sum_{i=1}^d 
\hloss_{t,i}.
  \end{split}
 \]
\end{theorem}
The proof of the theorem is deferred to Section~\ref{sec:analysis}.
Armed with this theorem, we are now ready to prove our first main result: a first-order bound on the expected regret of 
\fpltrix. 
\begin{corollary}\label{cor:nonadapt}
Consider \fpltrix run with the time-independent parameters $\gamma = \eta m$ and $\beta d = \gamma$ (and thus $B = 
\log(d/m) - \log \eta$). The expected regret of the resulting algorithm satisfies
 \[
 \begin{split}
  \EE{R_T} \le& \frac{mD}{\eta} + 3 \eta m d L_T^*
+3m^2 d(D+B) 
 + 3 d.
 \end{split}
 \]
 In particular, setting $\eta = \min\ev{1,\sqrt{\frac{3\log (d/m)+1}{dL_T^*}}}$ guarantees
 \[
 \begin{split}
  \EE{R_T} \le 5.2m\sqrt{dL_T^*\pa{\log (d/m)+1}} + 1.5 m^2 d \max\ev{\log(dL_T^*),0} + \OO\pa{m^2 d 
\log(d/m)}.
 \end{split}
 \]
\end{corollary}
\begin{proof}
 Let $\bv_* = \argmin_{\bv\in\Sw}\bv\transpose\bL_T$. The proof of the first statement follows directly from combining 
the bounds of Theorem~\ref{thm:key} and Lemma~\ref{lem:lossbound} for $\bv = \bv_*$, taking expectations and noticing 
that $\mathbb{E}\big[\bv_*\transpose\hbL_T]\le L_T^*$. For 
the second statement, first consider the case when $\eta=1$ and thus $\beta = \eta (m/d) = m/d$, giving $B = 
\log(1/\beta)  = \log (d/m)$. Now notice that the setting of $\eta$ implies
$ L_T^* \le (3D)/d$
and thus $L_T^* \le \sqrt{(3 L_T^* D)/d}$.
Then, substituting the value of $\eta$ into the first bound of the theorem gives
\[
 \begin{split}
  \EE{R_T} 
 \le& 3 m\sqrt{3 dL_T^* D}+ mD
+3m^2 d\bigl(2\log (d/m) + 1\bigr)  
 + 3 d,
 \end{split}
 \]
proving the statement as $3\sqrt{3}< 5.2$. For the case $\eta \le 1$, the bound follows from 
substituting the value of $\eta$ as
\[
 \begin{split}
  \EE{R_T} 
 \le& 2 m\sqrt{3 dL_T^* D} + \frac 32 m^2 d \log(dL_T^*) +3m^2 d\bigl(2\log (d/m) + 1\bigr) 
 + 3 d,
 \end{split}
 \]
 where we used that $B = \log(d/m) + \log(1/\eta)$ and $\log(1/\eta) \le \log(dL_T^*)/2$.
\end{proof}
Notice that achieving the above bounds requires \emph{perfect} knowledge of $L_T^*$, which is usually not available in 
practice. While one could use a standard doubling trick to overcome this difficulty, we choose to take a different 
path to circumvent this issue, and propose a modified version of \fpltrix that is able to tune its learning rate 
and other parameters solely based on observations. We note that our tuning rule has some unorthodox qualities and might 
be of independent interest.

Similarly to the parameter choice suggested by 
Corollary~\ref{cor:nonadapt}, we will use a single sequence of decreasing non-negative learning rates $\pa{\eta_t}$ 
and set $\gamma_t = m \eta_t$ and $\beta_t = (m/d) \eta_t$ for all $t$.
For simplicity, let us define the notations $s_t = \sum_{i=1}^d 
\hloss_{t,i}$ and $S_t = \frac 1D + \sum_{k=1}^{t} s_k$, with $S_0 = \frac 1D > 0$. With these notations, we define our 
tuning rule as
\begin{equation}\label{eq:tuning}
 \eta_t = \sqrt{\frac{D}{S_{t-1}}}.
\end{equation}
Notice that $\eta_1 = D$, and thus $\beta_1 = \eta_1 (m/d) = (m/d)(\log(d/m) + 1)< 1$, ensuring 
that $B_1>0$ and the algorithm is well-defined. This follows from the inequality $z(1-\log z)< 1$ that holds for all 
$z\in(0,1)$.
The delicacy of the tuning rule~\eqref{eq:tuning} is that the terms $s_t$ are themselves bounded in terms of the random 
quantity $1/\eta_t$, and not some problem-dependent constant. To the best of our knowledge, all previously known 
 analyses concerning adaptive learning rates apply a deterministic bound on $s_t$ at some point, largely simplifying 
the analysis. As we will see below, treating this issue requires a bit more care than usual.
 The following theorem presents the performance guarantees of the resulting variant of \fpltrix.
\begin{theorem}
The regret of \fpltrix with the adaptive learning rates defined in Equation~\eqref{eq:tuning} simultaneously satisfies
 \[
 \begin{split}
  \EE{R_T} \le& 13 m\sqrt{dL_T^* \pa{\log(d/m) + 1}}
  + \OO\pa{m^2 d \log(dT)}
 \end{split}
 \]
 and
 \[
 \begin{split}
  \EE{R_T} \le& 13 m\sqrt{dT \pa{\log (d/m) + 1}} + 9.49m.
 \end{split}
 \]
\end{theorem}
\begin{proof}
Let $\bv_* = \argmin_{\bv\in\Sw}\bv\transpose\bL_T$. First, notice that the learning-rate sequence defined by 
Equation~\eqref{eq:tuning} is nonincreasing as required by 
Theorem~\ref{thm:key}. Also note that $s_t$ is nonnegative and is bounded by $\frac{m}{\gamma_t} = \frac{1}{\eta_t}$ 
for all $t$, and
$ \frac{1}{\eta_t} = \sqrt{S_{t-1}/D} \le S_{t-1}$ holds
since $S_{t-1}\ge \frac 1D$ for all $t$.
These facts together imply that $\eta_t \le \sqrt{2D/S_t}$ as
\[
\begin{split}
 \sqrt{\frac{2D}{S_{t}}} =& \sqrt{\frac{2D}{S_{t-1} + s_t}}
 \ge  \sqrt{\frac{2D}{S_{t-1} + \frac{1}{\eta_t}}}
 \ge  \sqrt{\frac{2D}{S_{t-1} + S_{t-1}}} =  \sqrt{\frac{D}{S_{t-1}}} = \eta_t.
\end{split}
\]
Combining the above bound with Lemma~3.5 of \citet{aucbge02}, we get
\[
 \sum_{t=1}^T \eta_t s_t \le \sqrt{2D} \sum_{t=1}^T \frac{s_t}{\sqrt{S_t}} \le 2\sqrt{2 D S_{T}}.
\]
Using $\eta_T \le \sqrt{2D/S_T}$ again, the right-hand side can be further bounded as
$ 2\sqrt{2D S_{T}} \le \frac{4 D}{\eta_T}$ and the bound of Theorem~\ref{thm:key} applied for $\bv_*$ becomes
\begin{equation}\label{eq:interm}
 \sum_{t=1}^T \bV_t\transpose \bloss_t - \bv_*\transpose\hbL_T
 \le 13 m\frac{D}{\eta_T}.
\end{equation}

Now, we are ready to prove the second bound in the theorem. Notice that $\frac{D}{\eta_T} = \sqrt{D S_{T-1}} \le 
\sqrt{D 
S_T}$ holds by the tuning rule and
\begin{equation}\label{eq:stbound}
 \EE{S_T} = \frac 1D + \sum_{i=1}^d \EE{\hL_{T,i}} \le 1 + dT,
\end{equation}
where we used that $\mathbb{E}\hloss_{t,i}\le\loss_{t,i}\le 1$.
The statement then follows from plugging this bound into Equation~\eqref{eq:interm}, taking expectations and using 
Jensen's inequality.

Proving the first bound requires a bit more care. First, an application of Lemma~\ref{lem:lossbound} gives
\[
 S_T \le \frac 1D + d\pa{\bv_*\transpose\hbL_T + \frac{m(B_T+D) + \frac 1m}{\eta_T}}.
\]
Now recall that $\frac{D}{\eta_T} \le \sqrt{D S_T}$ holds by the tuning rule. Bounding $S_T$ as above, this implies
\[
 \frac{D}{\eta_T} \le \sqrt{1 + dD\pa{\bv_*\transpose\hbL_T + \frac{m(B_T+D) + \frac 1m}{\eta_T}}}~.
\]
Solving the resulting quadratic equation for the largest possible value of $1/\eta_T$ gives
\[
\begin{split}
 \frac{D}{\eta_T} \le& \sqrt{1 + dD \bv_*\transpose\hbL_T  } + 2md(\log(1/\beta_T)+D) + \frac{2d}{m}
 \\
 \le& \sqrt{dD \bv_*\transpose\hbL_T } + md(\log(S_T)+3\log(d/m) + 2) + \frac{2d}{m} + 1.
\end{split}
\]
The first term can be directly bounded by using Jensen's inequality as 
$\mathbb{E}\Bigl[\sqrt{\bv_*\transpose\hbL_T}\Bigr] \le 
\sqrt{\bv_*\transpose \bL_T} = \sqrt{L_T^*}$. 
Finally, we bound $\EE{\log(S_T)} \le \log(dT + 1)$ by using the inequality~$\eqref{eq:stbound}$.
The statement of the theorem now follows from substituting into Equation~\eqref{eq:interm} and taking expectations.
\end{proof}

\section{The proof of Theorem~\ref{thm:key}}\label{sec:analysis}
Finally, let us turn to proving our key theorem. 
For the proof, we recall the auxiliary forecaster defined in 
Equation~\eqref{eq:aux} that uses a fixed 
non-truncated perturbation vector $\tbZ$ and also define a variant that also allowed to peek one step into the future:
\[
  \tbV_t = \argmin_{\bv\in\Sw} \bv\transpose\pa{\eta_t \hbL_{t-1} - \tbZ} \quad\mbox{and}\quad \tbV^+_t = 
\argmin_{\bv\in\Sw} \bv\transpose\pa{\eta_t \hbL_{t} - \tbZ}.
\]
We will use the notation $\tp_{t}^+(\bv) = \PPcc{\tbV_t^+ = \bv}{\F_t}$ for all $\bv\in\Sw$.

We start with the following two standard statements concerning the performance of the auxiliary forecaster 
\citep{NeuBartok13}. Note that the first of these lemmas slightly improves on the result of \citet{NeuBartok13} in 
replacing their $\log d$ factor by $\log (d/m)$. For completeness, we provide the proof of this improved bound in the 
Appendix.
\begin{lemma}\label{lem:cheat}
For any $\bv\in\Sw$,
 \begin{equation}\label{eq:btl}
\begin{split}
\sum_{t=1}^T \sum_{\bu\in\Sw} \tp_t^+(\bu)\left(\bu - \bv\right)^\top \hbl_t \le \frac{m\left(\log
(d/m)+1\right)}{\eta_T}.
\end{split}
\end{equation}
\end{lemma}
\begin{lemma}\label{lem:price}
 For all $t$, 
\[
 \sum_{\bu\in\Sw} \bigl(\tp_t(\bu) - \tp_t^+(\bu)\bigr) \bu^\top \hbl_t \le \eta_t \sum_{\bu\in\Sw} \tp_t(\bu)
\left(\bu^\top \hbl_{t}\right)^2.
\] 
\end{lemma}
The following lemma bounds the term on the right-hand side of the above bound.
\begin{lemma}\label{lem:quad}
 Assume that $\beta d \le \gamma$. Then for all $t$,
 \[
  \sum_{\bu\in\Sw}\tp_t(\bu) \pa{\bu\transpose\hbl_t}^2 \le m \sum_{j=1}^d \hloss_{t,j}.
 \]
\end{lemma}
\begin{proof}
The statement is proven as 
\[
\begin{split}
\sum_{\bu\in\Sw}\tp_t(\bu) \pa{\bu\transpose\hbl_t}^2
&=
\EEcc{\sum_{i=1}^d\sum_{j=1}^d \left(\wt{V}_{t,i}\hloss_{t,i}\right)\cdot\left(\wt{V}_{t,j}\hloss_{t,j}\right)}{\F_{t}}
\le
\sum_{i=1}^d \frac{V_{t,i} \tq_{t,i}}{q_{t,i} + \gamma_t} \cdot 
\sum_{j=1}^d \hloss_{t,j}
\\&\le
\sum_{i=1}^d V_{t,i} \frac{q_{t,i} + \beta_t d}{q_{t,i} + \gamma_t} \cdot 
\sum_{j=1}^d \hloss_{t,j}
\le
m\sum_{j=1}^d \hloss_{t,j},
\end{split}
\]
where the first inequality follows from the definitions of $\hbl_t$ and $\tq_{t,i}$ and bounding $\wt{V}_{t,j}\le 1$, 
the second one follows from using Lemma~\ref{lem:ptrunc} 
and the last one from $\beta_t d \le \gamma_t$ and $\onenorm{\bV_t}\le m$.
\end{proof}
Our final lemma quantifies the bias of the learner's estimated losses.
\begin{lemma}
 For all $t$,
 \[
  \sum_{\bu\in\Sw} \tp_t(\bu) \pa{\bu\transpose\hbl_t} \ge \bV_t\transpose\bloss_{t} - \pa{\gamma_t + 
\beta_t d}\sum_{i=1}^d \hloss_{t,i}.
 \]
\end{lemma}
\begin{proof}
 First, note that 
by Lemma~\ref{lem:ptrunc}, we have
 \[
  \sum_{\bu\in\Sw} \tp_t(\bu) \pa{\bu\transpose\hbl_t} = \sum_{i=1}^d \tq_{t,i} \hloss_{t,i} 
  \ge \sum_{i=1}^d q_{t,i} \hloss_{t,i} - \beta_t d \sum_{i=1}^d \hloss_{t,i}.
 \]
 Then, the proof is concluded by observing that
 \[
  \begin{split}
   \sum_{i=1}^d q_{t,i} \hloss_{t,i}
   =& \sum_{i=1}^d q_{t,i} \frac{V_{t,i}\loss_{t,i}}{q_{t,i} + \gamma_t}
   = \bV_t\transpose\bloss_t - \gamma_t \sum_{i=1}^d \frac{V_{t,i}\loss_{t,i}}{q_{t,i} + \gamma_t}
   = \bV_t\transpose\bloss_t - \gamma_t \sum_{i=1}^d \hloss_{t,i}.
  \end{split}
 \]
\end{proof}
The statement of Theorem~\ref{thm:key} follows from piecing the lemmas together.

\section{Discussion}\label{sec:why}
We conclude by discussing some implications and possible extensions of our results.
\paragraph{Why truncate?}
One might ask whether truncating the perturbations is really necessary for our bounds to hold. We now provide 
an argument that shows that it is possible to achieve similar results \emph{without} explicit truncations, if we 
accept an additive $\OO(\log(dT))$ term in our bound. In particular, consider \fpl with non-truncated exponential 
perturbations. It is easy to see that with probability at least $1-\delta/(dT)$, all perturbations remain bounded by 
$B=\log(\frac{dT}{\delta})$. One can then analyze \fpl under this condition along the same lines as the proof of 
Corollary~\ref{cor:nonadapt}, the main difference being that 
we also have to account for the regret arising from 
the low-probability event that not all perturbations are bounded. Bounding the regret in this case by the trivial bound 
$dT$, this additional term becomes $\delta dT$. Setting $\delta = \sqrt{dL_T^*}/dT$ makes the total regret 
$\sqrt{dL_T^*}$---however, notice that this gives $B = \Theta(\log(dT))$, which shows up additively in the bound.
A similar argument can be shown to work for the adaptive version of \fpltrix. We note that the 
implicit exploration induced by the bias parameter $\gamma$ and other techniques developed in this paper are still 
essential to prove these results.

\paragraph{High-probability bounds.} Another interesting question is whether our results can be extended to hold with 
high probability. Luckily, it is rather straightforward to extend Corollary~\ref{cor:nonadapt} to achieve such a 
result by replacing $\hloss_{t,i}$ with $\tloss_{t,i} = \frac{1}{\omega} \log(1 + \omega \hloss_{t,i})$ for an 
appropriately chosen $\omega>0$, as suggested by \citet{AB10}. While such a result would also enable us to handle 
adaptive environments, it has the same drawback as Corollary~\ref{cor:nonadapt}: it requires perfect knowledge of 
$L_T^*$. Proving high-confidence bounds for the adaptive variant of \fpltrix, however, is far less straightforward; we 
leave this investigation for future work.

\section*{Acknowledgments}
This work was supported by INRIA, the French Ministry of Higher Education and Research, and by FUI project Herm\`es.
The author wishes to thank the anonymous reviewers for their valuable comments that helped to improve the paper.

\appendix

\section{Some technical proofs}
We first prove a statement regarding the mean of the sum of top $m$ out of $d$ independent exponential random variables.
\begin{lemma}\label{lem:expbound}
 Let $Z_1,Z_2,\dots,Z_d$ be i.i.d.~exponential random variables with unit expectation and let $Z_1^*,Z_2^*,\dots,Z_d^*$ 
be their permutation such that $Z_1^*\ge Z_2^*\ge\dots\ge Z_d^*$. Then, for any $1\le m\le d$, 
\[
 \EE{\sum_{i=1}^m Z_i^*} \le m\pa{\log\pa{\frac dm} + 1}.
\]
\end{lemma}
\begin{proof}
 Let us define $Y = \sum_{i=1}^m Z_i^*$. Then, as $Y$ is nonnegative, we have for any $A\ge 0$ that
 \[
  \begin{split}
   \EE{Y} =& \int_0^\infty \PP{Y>y}\,dy 
   \\
   \le & A + \int_A^\infty \PP{\sum_{i=1}^m Z_i^*>y}\,dy
   \\
   \le & A + \int_A^\infty \PP{Z_1^*>\frac ym}\,dy
   \\
   \le & A + d\int_A^\infty \PP{Z_1>\frac ym}\,dy 
   \\
   = & A + d e^{-A/m},
  \end{split}
 \]
 where the last inequality follows from the union bound.
 Setting $A = m\log (d/m)$ minimizes the above expression over the real line, thus proving the statement.
\end{proof}
With this lemma at hand, we are now ready to prove Lemma~\ref{lem:cheat}.
\begin{proof}[Proof of Lemma~\ref{lem:cheat}]
To enhance readability, define $\mu_t = 1/\eta_t$ for $t\ge 1$ and $\mu_0 = 0$. We
start by applying the
classical follow-the-leader/be-the-leader lemma (see, e.g.,
\citealp[Lemma~3.1]{CBLu06:book}) to the loss sequence defined as
$\bigl(\hbl_1-\mu_1\tbZ,\hbl_2 - (\mu_2 - \mu_1)\tbZ,\dots,\hbl_T - (\mu_T - \mu_{T-1}) \tbZ\bigr)$ to obtain
\[
 \sum_{t=1}^T \pa{\tbV_t^+}\transpose\pa{\hbl_t - \pa{\mu_t - \mu_{t-1}}\tbZ}
\le \bv\transpose\pa{\hbL_T - \mu_T\tbZ}.
\]
After reordering and observing that $-\bv\transpose\tbZ\le 0$, we get
\[
\begin{split}
 \sum_{t=1}^T \pa{\tbV_t^+ - \bv}\transpose \hbl_t &\le \sum_{t=1}^T (\mu_t -
\mu_{t-1})\pa{\tbV_t^+}\transpose\tbZ
 \\
 &\le \sum_{t=1}^T (\mu_t - \mu_{t-1}) \cdot \max_{\bu\in\Sw}\bu\transpose\tbZ =
\mu_T \cdot \max_{\bu\in\Sw}\bu\transpose\tbZ,
\end{split}
\]
where we used that the sequence $\pa{\mu_t}$ is nondecreasing and $\bu\transpose\tbZ\ge 0$ for all $\bu\in\Sw$.
The result follows from integrating both sides with respect to the distribution of $\tbZ$ and applying 
Lemma~\ref{lem:expbound} to obtain $\EE{\max_{\bu\in\Sw}\bu\transpose\tbZ}\le m\pa{\log(d/m)+1}$.
\end{proof}

\end{document}